\documentclass[letter,11pt]{article}

\usepackage[utf8]{inputenc}
\usepackage[T1]{fontenc}

\usepackage[a4paper,margin=1in]{geometry}

\usepackage[]{natbib}
\usepackage{url}            % simple URL typesetting
\usepackage{amsfonts}       % blackboard math symbols
\usepackage{nicefrac}       % compact symbols for 1/2, etc.
\usepackage{microtype}      % microtypography

\usepackage[page,header]{appendix}
\usepackage[usenames,dvipsnames]{xcolor}

\definecolor{DarkRed}{rgb}{0.368,0.097,0.078}
\definecolor{DarkBlue}{rgb}{0.2,0.2,0.6}

\usepackage[colorlinks = true,
    linkcolor = DarkBlue,
    anchorcolor = DarkBlue,
    citecolor = DarkBlue,
    filecolor = DarkBlue,
    urlcolor = DarkBlue,
    pagebackref]{hyperref}

\usepackage{booktabs,threeparttable,array,caption,makecell,graphicx}

\newcommand{\CellNote}[2]{#1\textsuperscript{#2}}

\newcommand{\expb}{2^{\Theta(n)}}

\newcommand{\dexpb}{2^{2^{\Theta(n)}}}

\newcommand{\RE}{\mathsf{RE}}
\newcommand{\REcap}{\mathsf{RE}(\cap)}
\newcommand{\REneg}{\mathsf{RE}(\neg)}

\newcommand{\dfa}{\mathsf{DFA}}   

\newcommand{\nfa}{\mathsf{NFA}} 

\newcommand{\err}{\texttt{Err}\xspace}

\newcommand{\re}{\text{RE}\xspace}
\newcommand{\ksat}{$k\text{-SAT}$\xspace}
\newcommand{\poly}{\text{poly}\xspace}

\newcommand{\boldz}{\boldsymbol{z}}
\newcommand{\boldx}{\boldsymbol{x}}
\newcommand{\boldy}{\boldsymbol{y}}
\newcommand{\boldu}{\boldsymbol{u}}

\newcommand{\dashrule}[4][\linewidth]{%
  \noindent\leavevmode
  \hbox to #1{\xleaders\hbox{\rule{#2}{#3}\hskip #4}\hfil}%
}

\usepackage{amsthm} % better to load after ams math
\usepackage{mathtools,thmtools}
\usepackage{amsfonts,amsmath,amssymb}
\usepackage[capitalise]{cleveref}
\usepackage{times}
\usepackage{algorithm}
\usepackage{algpseudocode}
\usepackage{thm-restate}
\usepackage{tcolorbox}
\usepackage{lipsum}
\usepackage{enumitem}
\usepackage{bm}
\usepackage{xspace}
\usepackage{nicefrac}
\usepackage{dsfont}
\usepackage{float}
\usepackage{bbm}
\usepackage{mdframed}
\usepackage{booktabs} 

\usepackage[color=green!25,prependcaption,textsize=tiny]{todonotes}

\usepackage{multirow}

\usepackage{array}

\usepackage{enumitem}
\usepackage{bm}
\usepackage{xspace}
\usepackage{nicefrac}
\usepackage{dsfont}

\declaretheoremstyle[
	    spaceabove=\topsep, 
	    spacebelow=\topsep, 
	    headfont=\normalfont\bfseries,
	    bodyfont=\normalfont\itshape,
	    notefont=\normalfont\bfseries,
	    notebraces={(}{)},
	    postheadspace=0.5em, 
	    headpunct={},
	    postfoothook=\noindent\ignorespaces
    ]{theorem}
\declaretheorem[style=theorem,numberwithin=section]{theorem}

\declaretheoremstyle[
	    spaceabove=\topsep, 
	    spacebelow=\topsep, 
	    headfont=\normalfont\bfseries,
	    bodyfont=\normalfont,
	    notefont=\normalfont\bfseries,
	    notebraces={(}{)},
	    postheadspace=0.5em, 
	    headpunct={},
	    postfoothook=\noindent\ignorespaces
    ]{definition}

\declaretheoremstyle[
        spaceabove=\topsep, 
        spacebelow=\topsep, 
        headfont=\normalfont\bfseries,
        bodyfont=\normalfont,
        notefont=\normalfont\bfseries,
        notebraces={}{},
        postheadspace=0.5em, 
        qed=$\blacksquare$, 
        headpunct={},
        postfoothook=\noindent\ignorespaces
    ]{proofstyle}
\declaretheorem[style=proofstyle,numbered=no,name=Proof]{proof}

\declaretheoremstyle[
        spaceabove=\topsep, 
        spacebelow=\topsep, 
        headfont=\normalfont\bfseries,
        bodyfont=\normalfont,
        notefont=\normalfont\bfseries,
        notebraces={}{},
        postheadspace=0.5em, 
        qed=$\blacksquare$, 
        headpunct={},
        postfoothook=\noindent\ignorespaces
    ]{proofstyle}
\declaretheorem[style=proofstyle,numbered=no,name=Proof sketch]{proofsketch}

\declaretheorem[style=theorem,sibling=theorem,name=Lemma]{lemma}

\declaretheorem[style=theorem,sibling=theorem,name=Assumption]{assumption}

\declaretheorem[style=theorem,numbered=no,name=Theorem]{theorem*}
\declaretheorem[style=theorem,numbered=no,name=Lemma]{lemma*}
\declaretheorem[style=theorem,numbered=no,name=Corollary]{corollary*}
\declaretheorem[style=theorem,numbered=no,name=Proposition]{proposition*}
\declaretheorem[style=theorem,numbered=no,name=Claim]{claim*}
\declaretheorem[style=theorem,numbered=no,name=Fact]{fact*}
\declaretheorem[style=theorem,numbered=no,name=Observation]{observation*}
\declaretheorem[style=theorem,numbered=no,name=Conjecture]{conjecture*}

\declaretheorem[style=definition,numbered=no,name=Definition]{definition*}
\declaretheorem[style=definition,numbered=no,name=Remark]{remark*}
\declaretheorem[style=definition,numbered=no,name=Example]{example*}
\declaretheorem[style=definition,numbered=no,name=Question]{question*}

\DeclareMathAlphabet{\mathbfsf}{\encodingdefault}{\sfdefault}{bx}{n}

\let\Pr\relax
\DeclareMathOperator{\Pr}{\mathbb{P}}

\newcommand{\lr}[1]{\mathopen{}\left(#1\right)}

\usepackage{mathtools}
\usepackage{float}
\usepackage{enumitem}
\usepackage{algorithm}
\usepackage{algpseudocode}

\newcommand{\cD}{\mathcal{D}}

\newcommand{\cX}{\mathcal{X}}

\newcommand{\cC}{\mathcal{C}}

\newcommand{\beq}{\begin{eqnarray*}}
\newcommand{\eeq}{\end{eqnarray*}}
\newcommand{\beqn}{\begin{eqnarray}}
\newcommand{\eeqn}{\end{eqnarray}}

\title{On the Hardness of Learning Regular Expressions}

\author{
Idan Attias\thanks{
Institute for Data, Econometrics, Algorithms, and Learning (IDEAL), hosted by the
University of Illinois at Chicago and Toyota Technological Institute at Chicago; \texttt{idanattias88@gmail.com}.}
\and Lev Reyzin\thanks{University of Illinois at Chicago; \texttt{lreyzin@uic.edu}.}
\and Nathan Srebro\thanks{Toyota Technological Institute at Chicago; \texttt{nati@ttic.edu}.}
\and Gal Vardi\thanks{Weizmann Institute of Science; \texttt{gal.vardi@weizmann.ac.il}.}
}

\date{}
\begin{document}
\maketitle

\begin{abstract}%
    Despite the theoretical significance and wide practical use of regular expressions, the computational complexity of learning them has been largely unexplored.
    We study the computational hardness of improperly learning regular expressions in the PAC model and with membership queries. 
    We show that PAC learning is hard even under the uniform distribution on the hypercube,
    and also prove hardness of distribution-free learning with membership queries. 
    Furthermore, if regular expressions are extended with complement or intersection, we establish hardness of learning with membership queries even under the uniform distribution.
    We emphasize that these results do not follow from existing hardness results for learning DFAs or NFAs, since the descriptive complexity of regular languages can differ exponentially between DFAs, NFAs, and regular expressions.
\end{abstract}

\section{Introduction}

What is the computational complexity of learning a regular expression (RE) of length at most $s$ over $\{0,1\}^n$, or over some other alphabet $\Sigma^n$, as a function of $n, s,$ and $|\Sigma|$? Can this be done in time $\poly(n,s)$? Surprisingly, despite the importance of regular expressions, this question has been left open and not explicitly addressed in the literature. In this paper, we establish hardness of learning regular expressions in the Probably Approximately Correct (PAC) and Membership Queries (MQ) models, both under arbitrary input distributions and under the uniform distribution. Our results are summarized in \Cref{table:results}. 

Learning regular languages has been extensively explored in the literature. In particular, some of the earliest results, both about proper learning in an inductive setting \citep{gold1978complexity,pitt1993minimum} and cryptographic hardness of improper PAC learning \citep{kearns1994cryptographic}, established the hardness of learning Deterministic Finite Automata (DFA). DFAs and REs are indeed equivalent in that both exactly characterize regular languages \citep{kleene1956representation}: any DFA has an equivalent RE, and every RE has an equivalent DFA. Why can’t we then conclude that since DFAs are hard to learn, i.e., regular languages are hard to learn, then REs are hard to learn? The important point is that when we say DFAs or REs are easy or hard to learn, we mean that it is easy or hard to learn languages with \emph{succinct} DFAs or REs. But even though every DFA has an equivalent RE and vice versa, the conversion may require exponential (or even super-exponential) blowups. These gaps are discussed in \Cref{sec:description-re-dfa} and summarized in \Cref{table:descriptional-complexity}. Therefore, even if we cannot learn in time polynomial in the size (or number of states) of a DFA, this does not contradict the possibility of learning in time polynomial in the length of an RE. Indeed, as we can see in \Cref{table:results}, there are differences in the complexity of learning DFAs vs. REs. In particular, for MQ learning, it is possible to learn 
%regular languages 
in time polynomial in the size of the DFA, but not polynomial in the length of the RE.

\begin{table}[H]
% \floatconts
  % {tab:results}
  {\caption{Tractability of learning regular expressions and automata. Our new hardness results are in bold font.}
  {%
    % \small
    \setlength{\tabcolsep}{4pt}%
    \renewcommand{\arraystretch}{1.15}%
    % Narrow first col; fixed widths avoid large inter-column gaps
    \resizebox{\linewidth}{!}{%
      \begin{tabular}{@{}>{\raggedright\arraybackslash}p{0.4\linewidth} *{4}{>{\centering\arraybackslash}p{0.18\linewidth}}@{}}
      \toprule
       & \multicolumn{2}{c}{\textbf{PAC}} & \multicolumn{2}{c}{\textbf{PAC + MQ}} \\
      \cmidrule(lr){2-3}\cmidrule(lr){4-5}
      \textbf{Representation complexity} & Dist.-free & Uniform dist. & Dist.-free & Uniform dist. \\
      \midrule
      $\poly(n)$ size $\mathsf{DFA}$           & \CellNote{Hard}{A,B,C,D} & \CellNote{Hard}{C} & Tractable & Tractable \\
      $\poly(n)$ size $\mathsf{NFA}$           & \CellNote{Hard}{A,B,C,D} & \CellNote{Hard}{C} & \CellNote{Hard}{A,\textbf{B,C,D}}      & ? \\
       $\poly(n)$ length $\mathsf{RE}$ (with only $\vee,\cdot,^{\star})$ & \CellNote{\textbf{Hard}}{A,B,C,D} & \CellNote{\textbf{Hard}}{C} & \CellNote{\textbf{Hard}}{B,C,D} & ? \\
      $\poly(n)$ length $\mathsf{RE}$ with intersection    & \CellNote{\textbf{Hard}}{A,B,C,D} & \CellNote{\textbf{Hard}}{A,C} & \CellNote{\textbf{Hard}}{A,B,C,D} & \CellNote{\textbf{Hard}}{A} \\
      $\poly(n)$ length $\mathsf{RE}$ with complement   & \CellNote{\textbf{Hard}}{A,B,C,D} & \CellNote{\textbf{Hard}}{A,C} & \CellNote{\textbf{Hard}}{A,B,C,D} & \CellNote{\textbf{Hard}}{A} \\
      \bottomrule
      \end{tabular}
    }

    \vspace{3pt}
    % Notes block (left-aligned bullets; use enumitem if you prefer)
    \begin{minipage}{\linewidth}
      \small %play with this size
      \begin{itemize}[leftmargin=1em,labelsep=0.4em,itemsep=1pt,topsep=0pt]
        \item[\textsuperscript{A}] Under one of the assumptions: RSA, factoring Blum integers, or quadratic residues.
        \item[\textsuperscript{B}] Assumption: random \ksat.
        \item[\textsuperscript{C}] Assumption: local-PRG.
        \item[\textsuperscript{D}] Assumption: sparse learning-parity-with-noise.
      \end{itemize}
\vspace{-5pt}\dashrule[\linewidth]{6pt}{0.4pt}{4pt} \vspace{-5pt}
      \begin{itemize}[leftmargin=1em,labelsep=0.4em,itemsep=1pt,topsep=0pt]
        \item \underline{DFA}: \citet{kearns1994cryptographic} showed distribution-free PAC hardness under Assumption~A; \citet{daniely2016complexity} established it under Assumption~B; and \citet{daniely2021local} under Assumption~C. Moreover, \citet{daniely2021local} proved PAC hardness on the uniform distribution (again under Assumption~C). There is a PAC-preserving reduction from DNFs to DFAs \citep{pitt1990prediction}; therefore, the hardness of learning DNFs under Assumption~D proved by \citet{applebaum2025structured} immediately implies the hardness of learning DFAs.
        In contrast, DFAs are tractably learnable with membership queries via the $L^\star$ algorithm of \citet{angluin1987learning}.
        \item\underline{NFA}: PAC hardness follows directly from the hardness of learning DFAs. 
        Moreover, for distribution-free learning with membership queries, \citet{angluin1991wont} proved hardness under Assumption A, and since REs are expressible using NFAs with polynomial blowup, our results for REs imply hardness under Assumptions B,C,D. 
        \item \underline{Plain RE}: We establish distribution-free hardness in PAC 
        under assumptions B,C,D (see 
        \Cref{thm:hardness-from-DNF}). \citet{angluin2013learnability} proved hardness with a \emph{non-binary} alphabet of a subclass of regular expressions, called shuffle ideals, under assumption A. 
        We claim that their construction can be modified to work for plain REs and the binary alphabet, see discussion in \Cref{sec:comparison-shuffle}.\\
        In \Cref{thm:re-pac-uniform} we prove PAC hardness under the uniform distribution, relying directly on Assumption~C.\\
        In Theorem~\ref{thm:dist-free-queries} we show distribution-free hardness of learning with queries under Assumptions B,C,D (and the existence of non-uniform one-way functions).
        \item \underline{RE with intersection or complement}: %PAC hardness follows directly %
        First, all hardness results for plain REs also apply here. Second, 
        in \Cref{thm:ere-pac+mq-hardness}, we prove hardness with membership queries on the uniform distribution under Assumption~A.
        \item \underline{Open Questions}:
        As far as we are aware, it remains unknown whether there is a tractable algorithm for learning plain REs or NFAs with membership queries under the uniform distribution. 
      \end{itemize}
      \vspace{-5pt}
      \noindent\rule{\linewidth}{0.4pt}
    \end{minipage}
  }%
  \label{table:results}
}
\end{table}

Learning regular languages, and REs in particular, is of practical as well as conceptual importance. Many rules in everyday data processing, manipulation, filtering, and analysis tasks are specified using regular expressions (e.g. validating numbers, extracting street names from addresses, or rewriting names as last–comma–first). REs are at the core of many languages and tools focused on data processing (e.g., grep, awk, and Perl), and are important elements of many others (e.g., Excel, Python, and almost all modern programming languages and shells). Thinking of machine learning as replacing expertly specified rules with learning from examples, being able to learn an RE from examples is thus a very natural task, and indeed a practical feature (e.g., given a column of addresses and a few examples of the house number, the machine should learn to extract the rest). Importantly, in all these languages, regular languages are specified by short REs, not DFAs. In addition to the classical REs as specified by Kleene, which support union, concatenation, and Kleene-star operations, several extensions to REs have also been suggested and are commonly supported by programming languages and libraries. We discuss these in \Cref{sec:re-regular-lang}.

Learning Regular Languages has also recently received renewed attention, with the study of what Regular Languages transformers can represent and learn (e.g., \citet{strobl2024formal} and the references therein). Indeed, this has been the direct impetus for our study.  Part of the goal of this paper is to emphasize the difference between different representations of Regular Languages, and in particular between DFAs and REs, to highlight the importance of studying learnability quantitatively as a function of a specific description language, and to encourage the study of learnability in terms of Regular Expression length.
%

%%%%%%%%%%%%%%%%%%%

\section{REs and Regular Languages}\label{sec:re-regular-lang}

We focus on the binary alphabet $\Sigma=\{0,1\}$. $\Sigma^\star$ defines the set of all binary strings, where $\Sigma^\star=\cup_{n\in \mathbb{N}}\Sigma^n$, with $\mathbb{N}$ including zero. 
\paragraph{Regular Expressions (RE).}
The set of regular expressions over an alphabet $\Sigma$, denoted $\RE(\Sigma)$, is the minimal set defined as follows: $\emptyset$, $\epsilon$, and every symbol $a \in \Sigma$ belong to $\RE(\Sigma)$. If $R_1, R_2 \in \RE(\Sigma)$, then the expressions
$(R_1 \vee R_2)$ (union), $(R_1 R_2)$ (concatenation, also denoted $(R_1 \cdot R_2)$), and $(R_1)^{\star}$ (Kleene star) are also in $\RE(\Sigma)$.

For $R \in \RE(\Sigma)$ the language it represents, written $L(R)\subseteq\Sigma^{\star}$, is defined inductively:  $L(\emptyset)=\varnothing$, $L(\epsilon)=\{\epsilon\}$, $L(a)=\{a\}$ for $a\in\Sigma$. $L(R_1\vee R_2)=L(R_1)\cup L(R_2)$, $L(R_1 R_2)=\{xy:x\in L(R_1),y\in L(R_2)\}$, and $L(R^{\star})=\bigcup_{k \in \mathbb{N}}L(R)^k$, where $L(R)^0=\{\epsilon\}$ and $L(R)^{k+1}=L(R)^kL(R)$. 
A language $L\subseteq\Sigma^{\star}$ is called \emph{regular} if $L=L(R)$ for some $R \in \RE(\Sigma)$. We use the shorthand $\RE$ for $\RE(\Sigma)$. 

\paragraph{Extended Regular Expressions.}
Having fixed the standard syntax and semantics of regular expressions, we now extend the language by adding new standard operators.

\vspace{0.1cm}
\emph{Intersection operator:}
The set $\REcap$ is obtained by closing $\re$ under intersection:
if $R_1,R_2 \in \re(\cap)$, then $(R_1 \wedge R_2) \in \re(\cap)$.  
The semantics is defined by $L(R_1 \wedge R_2) = L(R_1) \cap L(R_2).$

\vspace{0.1cm}
\emph{Complement (negation) operator:}
The set $\re(\neg)$ is obtained by closing $\re$ under complement:
if $R \in  \re(\neg)$, then $(\neg R) \in \re(\neg)$.  
The semantics is defined by $L(\neg R) = \Sigma^\star \setminus L(R).$

\vspace{0.1cm}
\emph{Counting operator:}
We also extend the language with a bounded-iteration (counting) operator, an operator supported in practice by tools such as egrep \citep{hume1988tale}, Perl regular-expression syntax \citep{wall1999programming}, and the XML Schema specification \citep{xml_schema}.  
For a regular expression $R$ and an integer $k\in\mathbb{N}$, we write the counting expression as $R^{k}$.
Its semantics is given by
$L\!\bigl(R^{k}\bigr) = L(R)^{k}$,
that is, $R^{k}$ matches any word that can be decomposed into $k$
contiguous factors, each belonging to $L(R)$.
We denote the resulting class of \emph{counting regular expressions} by $\RE(\#)$. 
For related variants of the definition, succinctness results, and associated computational problems, see \cite{meyer1972equivalence,kilpelainen2003regular,gelade2012regular,gelade2010succinctness}.
We use the notation $R^k$ as shorthand even when explicit counting is not permitted. In such cases, it should be understood as repeated concatenation. When counting is allowed, the encoding simply becomes more succinct.

% \vspace{0.1cm}
\paragraph{Size of regular expressions.}
We define the size of a (possibly extended) regular expression $R$, denoted by $\lvert R\rvert$, as the
total number of atomic and operators that appear in the concrete syntax of $R$:
$\lvert \varnothing\rvert = 1$, $\lvert \varepsilon\rvert = 1$, and $\lvert a\rvert = 1$ for $a\in\Sigma$.
$\lvert R_1 R_2\rvert = \lvert R_1\rvert+\lvert R_2\rvert$, 
    $\lvert R_1 \vee R_2\rvert = \lvert R_1\rvert+\lvert R_2\rvert+1$, 
    $\lvert r^{\star}\rvert = \lvert r\rvert+1$, $\lvert R_1 \wedge R_2\rvert = \lvert R_1\rvert+\lvert R_2\rvert+1$, $\lvert \lnot R \rvert = \lvert R \rvert +1$,
    and $\lvert r^{k}\rvert = \lvert r\rvert+\lceil\log(k)\rceil$, where counting is allowed and $k$ is encoded in binary, and otherwise  $\lvert r^{k}\rvert = k\lvert r\rvert$.
For alternative size measures (all equivalent up to a polynomial factor) and detailed succinctness results, see
\cite{ellul2005regular,gelade2010succinctness,gruber2015finite}.

\paragraph{DFAs and NFAs.}
A Nondeterministic Finite Automaton (NFA) is a tuple 
$A=(\Sigma,Q,Q_0,\delta,F)$, where $\Sigma$ is a finite alphabet, 
$Q$ is a finite set of states, $Q_0 \subseteq Q$ is the set of initial states, 
$\delta: Q \times \Sigma \to 2^Q$ is the transition function, and $F \subseteq Q$ 
is the set of final states. Given a word $w = \sigma_1 \cdots \sigma_\ell \in \Sigma^*$, 
a run of $A$ on $w$ is a sequence of states $r = q_0, \ldots ,q_\ell$ such that 
$q_0 \in Q_0$ and $q_{i+1} \in \delta(q_i,\sigma_{i+1})$ for all $i \geq 0$. 
The run $r$ is accepting if $q_\ell \in F$. We say that $A$ accepts $w$ if 
it has an accepting run on $w$, and we denote by $L(A)$ the language of $A$, 
namely the set of words it accepts. When $|Q_0|=1$ and 
$|\delta(q,\sigma)| \leq 1$ for all $q \in Q$ and $\sigma \in \Sigma$, then $A$ 
is deterministic. In this case we say that $A$ is a Deterministic Finite Automaton (DFA).
\subsection{Equivalence and gaps between REs and automata}\label{sec:description-re-dfa}

In this section, we review the vast literature on the succinctness of representation
among different variants of REs, DFAs,
and NFAs, see \Cref{table:descriptional-complexity}.
A central claim of this paper is that the \emph{description length} of objects
representing regular languages plays a crucial role in determining their \emph{tractable learnability}.

We compare these models by analyzing their description lengths, i.e., the length of a binary
string encoding the object. For DFAs with $q$ states (over a fixed alphabet), the description
length is upper bounded by $O(q\log q)$. For NFAs with $q$ states, a dense encoding of
transition sets yields $O(q^2)$ bits. Although one may alternatively take the number
of states as the size parameter (a common convention), this differs only by a polynomial factor.
Importantly, the choice of size measure affects the analysis of conversions: for example,
the transformation from a DFA to an NFA is linear in the number of states, whereas when measured
in terms of description length the target can scale as $\Theta(q^2)$ rather than $\Theta(q\log q)$.

For regular expressions, the description length is defined as the size of the expression,
namely, the number of symbols and operators it contains.
Parentheses can be disregarded, as they only affect the size by a constant multiplicative
factor.\footnote{For a discussion of size definitions and their implications see
\cite{ellul2005regular,gelade2010succinctness,gruber2015finite}.}
Thus, for a constant-size alphabet $\Sigma$ and a constant operator set $\Gamma$
(including $\varepsilon$ and $\emptyset$), an RE of length $\ell$
has description length $\ell \cdot \big\lceil \log\big(|\Sigma|+|\Gamma|+2\big)\big\rceil = O(\ell)$.
If the syntax is extended with counting operators, then each integer up to $N$
requires $O(\log N)$ bits to encode, and the resulting description length becomes $O(\ell \log N)$.

\begin{table}[h]
\captionsetup{justification=raggedright,singlelinecheck=false}
\caption{Succinctness/translation blow-ups among $\dfa$, $\nfa$, $\RE$, $\REcap$, and $\REneg$.}
\label{table:descriptional-complexity}
\setlength{\tabcolsep}{3pt}%
\renewcommand{\arraystretch}{1.15}%
\resizebox{\linewidth}{!}{%
\begin{tabular}{cccccc}
\toprule
\textbf{Source $\to$ Target} & $\dfa$ & $\nfa$ & $\RE$  & $\REcap^\dag$ & $\REneg^\dag$\\
\midrule
\begin{tabular}[c]{@{}c@{}}
    $\dfa$ \\
    (description length  or $\#$states)
  \end{tabular}
  & --
  & $\poly(n)$
  & \begin{tabular}[c]{@{}c@{}}$\expb$ \\
  \cite{ehrenfeucht1976complexity}\\ \cite{gruber2015finite}\end{tabular}
  & $2^{O(n)}$
  & $2^{O(n)}$ \\
\midrule
\begin{tabular}[c]{@{}c@{}}
    $\nfa$ \\
    (description length  or $\#$states)
  \end{tabular}
  & \begin{tabular}[c]{@{}c@{}}%
  $2^{\Theta(n)}$\\
   \cite{rabin1959finite}%
   \\
   \cite{moore1971bounds,meyer1971economy}
\end{tabular}
  & --
  & \begin{tabular}[c]{@{}c@{}}$\expb$ \\
  \cite{gruber2015finite}\\ \cite{ehrenfeucht1976complexity}\end{tabular}
  & $2^{O(n)}$
  & $2^{O(n)}$ \\
\midrule
\begin{tabular}[c]{@{}c@{}} 
    $\RE$ \\ (length)
  \end{tabular}
  & \begin{tabular}[c]{@{}c@{}}$\expb$ \\
  \cite{gruber2015finite}\\ \cite{ehrenfeucht1976complexity}\end{tabular}
  & \begin{tabular}[c]{@{}c@{}}$\poly(n)$\\ \cite{thompson1968programming}\\ 
  \cite{gruber2015finite}\end{tabular}
  & --
  & $\poly(n)$
  & $\poly(n)$ \\
\midrule
\begin{tabular}[c]{@{}c@{}} 
    $\REcap$ \\ (length)
  \end{tabular}
  & \begin{tabular}[c]{@{}c@{}} $\dexpb$ \\
  \cite{gelade2010succinctness}
  \end{tabular}
  & \begin{tabular}[c]{@{}c@{}} $\expb$ \\
  \cite{gelade2010succinctness}
  \end{tabular}
  & \begin{tabular}[c]{@{}c@{}} $\dexpb$ \\
  \cite{gelade2012succinctness}
  \end{tabular}
  & --
  & $\poly(n)$
   \\
\midrule
\begin{tabular}[c]{@{}c@{}} 
    $\REneg$ \\ (length)$^*$
  \end{tabular}
  & \begin{tabular}[c]{@{}c@{}} $\underbrace{2^{2^{\cdot^{\cdot^{2}}}}}_{\Theta(n)}$ \\
  \cite{dang1973complexity}\\
  \cite{stockmeyer1973word}
  \end{tabular}
  & \begin{tabular}[c]{@{}c@{}} $\underbrace{2^{2^{\cdot^{\cdot^{2}}}}}_{\Theta(n)}$ \\
  \cite{dang1973complexity}\\
  \cite{stockmeyer1973word}
  \end{tabular}
  &  \begin{tabular}[c]{@{}c@{}} $\underbrace{2^{2^{\cdot^{\cdot^{2}}}}}_{\Theta(n)}$ \\
  \cite{dang1973complexity}\\
  \cite{stockmeyer1973word}
  \end{tabular}
  &  $\underbrace{2^{2^{\cdot^{\cdot^{2}}}}}_{O(n)}$ 
  & -- \\
\bottomrule
\end{tabular}%
}
\\[0.3em]  % vertical space
\footnotesize
$^*$ Depth-sensitive succinctness for conversions from RE($\neg$) to RE/DFA/NFA:\\
The complexity of the conversion grows with the complement-nesting depth of the regular expression. Thus, with unbounded depth the blow-up is non-elementary. For a single negation (depth~1), translating back to a plain RE already requires $2^{2^{\Theta(n)}}$ symbols \citep{gelade2012succinctness}. See also~\citet{gelade2010succinctness} for the non-elementary succinctness of RE($\neg$) versus automata.
\\
$\dag$ We are unsure about the tightness of the non-polynomial upper bounds in these columns.
\\
\vspace{-5pt}\noindent\rule{\linewidth}{0.4pt}
\end{table}

\section{Learning Models}
A concept class $\cC$ is a series of collections of functions $\cC_n\subseteq \{0,1\}^{\cX_n}$, for $n\in\mathbb{N}$, where $\cX_n=\{0,1\}^n$. Sometimes we abuse the notation and identify $\cC$ with $\cC_n$ where the meaning is clear from the context.

\paragraph{PAC learning \citep{valiant1984theory}.}
A (possibly randomized) algorithm $A$ is said to \emph{PAC learn} $\cC_n$ if, for all $\epsilon, \delta \in (0,1)$, for all $n \in \mathbb{N}$, for all $c^\star \in \cC_n$, and for all distributions $\cD_n$ over $\cX_n$, the following holds:  
Given access to i.i.d.\ examples $(x, c^\star(x))$ drawn from $\cD_n$, the algorithm $A$ outputs a (description of) hypothesis $h$ such that, with probability at least $1-\delta$ (over the random draw of examples and the internal randomness of $A$), $\Pr_{x \sim \cD_n} \big[ h(x) \neq c^\star(x) \big] \le \epsilon.$ 

We say that $\cC_n$ is \emph{tractably learnable} if there exists a learning algorithm $A$ and a polynomial $p$ such that, for every $n, \epsilon, \delta$, every distribution $\cD_n$, and every $c^\star \in \cC_n$, the algorithm $A$ PAC learns $\cC_n$ and runs in time at most $p(n,1/\epsilon,1/\delta)$.

The learner may output an arbitrary hypothesis $h : \cX_n \to \{0,1\}$ (not necessarily in $\cC_n$), provided that $h$ can be evaluated in polynomial time in $(n,1/\epsilon,1/\delta)$. This setting is also known as improper learning (or representation-independent learning).

\paragraph{Learning with membership queries \citep{angluin1987learning,angluin1988queries}.}
In the membership query (MQ) model, the learner is given the additional ability to actively request the label of any instance of its choice. Formally, besides receiving i.i.d.\ labeled examples $(x, c^\star(x))$ drawn from $\cD_n$, the learner may adaptively query an oracle on any $x \in \cX_n$ to obtain $c^\star(x)$. A learning algorithm $A$ is said to \emph{PAC+MQ learn $\cC_n$} if it satisfies the same accuracy and confidence guarantees as in the PAC model. Tractability is defined in the same way as for PAC learning.

 We also consider $\gamma$-weak learnability, where the error of the hypothesis returned by the learning algorithm is just slightly better than a random guess, that is, $\Pr_{x \sim \cD_n} \big[ h(x) \neq c^\star(x) \big]\leq 1/2-\gamma$, where $\gamma>0$ is either a constant or $\frac{1}{\poly(n)}$.
Computational hardness results for improper learning in a distribution-free setting are translated to the hardness of
improper weak learning, since, by using boosting algorithms \cite{schapire1990strength,freund1995boosting,schapire2013boosting}, if there is an efficient algorithm with error at most $\frac{1}{2}-\frac{1}{\poly(n)}$, then there is an efficient boosting algorithm that learns $\cC_n$. When we consider hardness of weak learning without specifying $\gamma$ we mean that there is no tractable $\gamma$-weak learning algorithm for any inverse-polynomial $\gamma$.

%%%%%%%%%%%%%%%%%%%%%%%%%%%%%%%%%%%%%%%%%%%%%%%%%%%%%%%%%%%%%%%%%%%

\section{Hardness of PAC Learning}\label{sec:hardness-re-pac}

In this section, we prove the hardness of PAC learning REs under several standard assumptions, both under arbitrary input distributions and under the uniform distribution. 

\subsection{Distribution-free hardness} \label{sec:dist-free-PAC}

We start with a distribution-free hardness result, obtained via an efficient
prediction-preserving reduction from PAC learning Disjunctive Normal
Forms (DNFs) to learning regular expressions. 
Then, 
known improper hardness results for DNFs in the PAC model directly implies the improper hardness of learning regular
expressions.

We define DNFs as follows. Fix a set of Boolean variables
$X=\{x_1,\ldots,x_n\}$. A literal is either a variable $x_i$ or its
negation $\lnot x_i$. A term (or clause) is a conjunction of one or
more distinct literals, $T = \ell_{i_1}\wedge \ell_{i_2}\wedge\cdots\wedge \ell_{i_k}$,
where each $\ell_{i_j}$ is a literal over $X$ and no variable appears twice
in $T$. The width of $T$ is $k$, and can be at most $n$.
A DNF formula is a disjunction of terms,
$\Phi = T_{1} \vee T_{2} \vee \cdots \vee T_{m}$,
where $m$ is a function of $n$.
The size of $\Phi$ is defined as its total number of literals,
$|\Phi|=\sum_{j=1}^{m}|T_j|$,
with $|T_j|$ equal to the width of $T_j$.
For an assignment $\boldx\in\{0,1\}^{n}$, the value $\Phi(\boldx)\in\{0,1\}$ is
obtained under the usual Boolean semantics.

\begin{lemma}[Polynomial-Size DNF $\to$ RE Translation]\label{lem:dnf-to-re}
Let $\Phi:\{0,1\}^n\to\{0,1\}$ be a DNF with $m(n)$ terms.
There is a plain regular expression $R_\Phi$ over $\{0,1\}$ of size $O(mn)$ such that
for every $\boldx\in\{0,1\}^n$, $\boldx \in L(R_\Phi) \leftrightarrow \Phi(\boldx)=1$.
\end{lemma}
\begin{proof}
Write $\Phi=\bigvee_{j=1}^m T_j$, where each term $T_j$ is a conjunction of
literals over $\{x_1,\ldots,x_n\}$.
For a given term $T_j$, define a regular expression $R_{T_j}=\gamma^{(j)}_1\gamma^{(j)}_2\cdots\gamma^{(j)}_n$,
where for each position $i$,
\[
\gamma^{(j)}_i =
\begin{cases}
1 & \text{if the literal } x_i \text{ occurs in } T_j,\\
0 & \text{if the literal } \neg x_i \text{ occurs in } T_j,\\
(0\vee 1) & \text{if } x_i \text{ does not occur in } T_j.
\end{cases}
\]
Set $R_\Phi = \big(R_{T_1} \vee \cdots \vee R_{T_{m(n)}}\big)$.
By construction, $R_{T_j}$ matches exactly those length-$n$ strings that satisfy
the constraints of $T_j$, and $R_\Phi$ matches a string iff at least one term holds,
i.e., iff $\Phi(\boldx)=1$.  The size is $O(n)$ per term, so $|R_\Phi|=O(mn)$.

For example, let $\Phi= \lr{x_1\wedge \lnot x_3 \wedge x_4}\vee \lr{\lnot x_2 \wedge x_3}$ be a DNF over variables $\{x_1,x_2,x_3,x_4\}$. The Equivalent RE would be 
$R_\phi= \lr{1(0\vee 1)01} \vee  \lr{(0\vee 1)01(0\vee 1)}$.
\end{proof}

The above lemma shows that poly-sized DNFs can be expressed by poly-sized REs. Hence, hardness of learning DNFs implies hardness of learning REs.
Thus, we establish the hardness of learning REs via known results on DNFs. 

Distribution-free hardness of PAC learning DNFs was established under the random $K$-SAT assumption \citep{daniely2016complexity}. 
Informally, this assumption states that 
no polynomial-time algorithm can efficiently refute (certify unsatisfiability of) typical random $K$-SAT formulas.
Under the local pseudorandom generator (local PRG) assumption, tighter distribution-free hardness and distribution-specific hardness for a product distribution (distinct from the uniform distribution) were established in \cite{daniely2021local}. 
See \Cref{sec:re-uniform} for the formal definition of local PRGs, where we use it explicitly to show hardness of learning REs under the uniform distribution. 
Finally, under a variant of the sparse learning parity with noise assumption, \citet{applebaum2025structured} proved distribution-free hardness of learning DNFs. 
Informally, this assumption states that even when each linear equation involves only a few variables, it is still computationally hard to recover the hidden parity in the presence of random noise.

\begin{theorem}[Distribution-Free Hardness of PAC Learning REs]\label{thm:hardness-from-DNF} 
Assuming either the local-PRG assumption (Assumption~\ref{ass:local-prg}), the random $K$-SAT assumption \citep{daniely2016complexity}, or the sparse learning-parity-with-noise assumption from \citet{applebaum2025structured}, 
for any constant $\epsilon>0$, there is no tractable algorithm that weakly PAC learns the concept class $$  \cC_n = \bigl\{ c : \{0,1\}^n \to \{0,1\} \bigm| c \text{ is representable by a regular expression of size}\leq n^\epsilon \bigr\}.$$
\end{theorem}
\begin{proof}
Under the local-PRG assumption \citep{daniely2021local}, or the hardness assumption for sparse learning parity with noise \citep{applebaum2025structured}, weakly learning DNFs with $m(n)=\omega(1)$ terms is hard.
Take $m(n)=\log(n)$, then by Lemma~\ref{lem:dnf-to-re}, it is hard to weakly learn REs of size $O(n\log(n))$. 
By a standard scaling argument, it implies that for any constant $\epsilon>0$, it is hard to learn REs of size $n^\epsilon$. Indeed, define $\tilde n = n^{1+1/\epsilon}$, and note that $\tilde n ^\epsilon= n^{1+\epsilon}$ is larger than $O(n\log(n))$ for a sufficiently large $n$. 
For an example $(\boldx,y) \in \{0,1\}^n \times \{0,1\}$ realizable by an RE of size $O(n \log(n))$, by padding the input $\boldx$ with zeros, we obtain $\tilde \boldx \in \{0,1\}^{\tilde n}$ where $(\tilde \boldx,y)$ is realizable by an RE of size at most $\tilde{n}^\epsilon$.
Thus, 
it is hard to weakly learn REs over $\{0,1\}^{\tilde n}$ of size $\tilde n ^\epsilon$.

Under the $K$-SAT assumption, weakly learning DNFs with $m(n)=\omega(\log(n))$ terms is hard \citep{daniely2016complexity}. Take $m(n)=\log^2(n)$, then by Lemma~\ref{lem:dnf-to-re}, it is hard to learn REs of size $O(n\log^2(n))$.
By a similar scaling argument, we get that it is hard to weakly learn REs of size $n^\epsilon$.
\end{proof}

\subsection{Hardness under the uniform distribution}\label{sec:re-uniform}

For the uniform distribution, the DNF route is insufficient: it is open whether DNFs are tractably learnable under the uniform distribution, and the best known algorithms are quasi-polynomial \citep{verbeurgt1990learning,linial1993constant}.

Therefore, we follow the approach introduced by \cite{daniely2021local}, showing that if regular expressions could be weakly learned tractably, then one could distinguish random strings from pseudorandom strings generated by a local pseudorandom generator. Their framework was used to prove distribution-free and distribution-specific hardness of PAC learning for several important concept classes, including DFAs. However, for our purpose, we need a construction tailored specifically to regular expressions.

\paragraph{Local pseudorandom generators (PRGs).}
An $(n,m,k)$-hypergraph is a hypergraph over vertex set $[n]$ with $m$ (ordered) hyperedges $E_1,\dots,E_m$, each of cardinality $k$, where all vertices in a hyperedge are distinct. We denote by $\mathcal{G}_{n,m,k}$ the distribution obtained by selecting each hyperedge independently and uniformly from all $n \cdot (n-1) \cdots (n-k+1)$ possible ordered $k$-tuples.

Let $P:\{0,1\}^k \to \{0,1\}$ be a predicate, and let $G$ be an $(n,m,k)$-hypergraph. Goldreich’s pseudorandom generator (PRG) \citep{goldreich2000candidate} is defined as
\[
f_{P,G} : \{0,1\}^n \to \{0,1\}^m, \qquad 
f_{P,G}(\boldx) = \bigl(P(\boldx|_{E_1}), \dots, P(\boldx|_{E_m})\bigr).
\]
The integer $k$ is called the \emph{locality} of the PRG. If $k$ is constant, the PRG is said to be \emph{local}. The PRG has \emph{polynomial stretch} if $m = n^s$ for some constant $s>1$. We denote by $\mathcal{F}_{P,n,m}$ the family of functions $f_{P,G}$ where $G$ ranges over $(n,m,k)$-hypergraphs. Sampling from $\mathcal{F}_{P,n,m}$ means choosing $G \sim \mathcal{G}_{n,m,k}$. We write $G \xleftarrow{R} \mathcal{G}_{n,m,k}$ for a random choice of $G$, and $\boldx \xleftarrow{R} \{0,1\}^n$ for a uniformly random string.

The family $\mathcal{F}_{P,n,m}$ is an $\varepsilon$-pseudorandom generator (PRG) if for every probabilistic polynomial-time algorithm $A$, the distinguishing advantage
\[
\left|
\Pr_{G \xleftarrow{R} \mathcal{G}_{n,m,k}, \boldx \xleftarrow{R} \{0,1\}^n}
  \bigl[A(G,f_{P,G}(\boldx))=1\bigr]
-
\Pr_{G \xleftarrow{R} \mathcal{G}_{n,m,k}, \boldy \xleftarrow{R} \{0,1\}^m}
  \bigl[A(G,\boldy)=1\bigr]
\right|
\]
is at most $\varepsilon$.

\begin{assumption}[Local PRG Assumption]\label{ass:local-prg}
For every constant $s>1$, there exists a constant $k$ and a predicate $P:\{0,1\}^k \to \{0,1\}$ such that $\mathcal{F}_{P,n,n^s}$ is a $\tfrac{1}{3}$-PRG.
\end{assumption}

Requiring a constant distinguishing advantage is weaker than the more common requirement of negligible advantage (see, e.g., \cite{applebaum2016cryptography,applebaum2016pseudorandomness,couteau2018pseudorandomness}). 
Local PRGs with polynomial stretch have been extensively studied and applied, for example in secure computation with constant overhead and in general-purpose obfuscation via constant-degree multilinear maps \citep{ishai2008secure,applebaum2017low,lin2016indistinguishability,lin2016generalized}. 
A key theoretical foundation was provided by \citet{applebaum2013cryptographic}, who showed that the above assumption holds if there exists a sensitive local predicate $P$ such that $\mathcal{F}_{P,n,n^s}$ is one-way, a variant of Goldreich’s original one-wayness assumption \citep{goldreich2000candidate}.
Assumption~\ref{ass:local-prg} was used by \citet{daniely2021local,daniely2023computational} for proving hardness of learning of various classes.

\begin{theorem}[Hardness of PAC Learning REs under the Uniform Distribution]\label{thm:re-pac-uniform}
Under the local-PRG assumption (Assumption~\ref{ass:local-prg}), for any constants $\epsilon,\gamma>0$, there is no tractable algorithm that $\gamma$-weakly learns the concept class 
\[\cC_n = \bigl\{ c : \{0,1\}^n \to \{0,1\} \bigm| c \text{ is representable by a regular expression of size}\leq n^\epsilon \bigr\}\]
on the uniform distribution over $\{0,1\}^n$.
\end{theorem}
We note that the hardness result holds even for regular expressions using only union and concatenation (i.e., without the Kleene star), provided that we bound the RE size by $O(n)$ (instead of $n^\epsilon$).

The formal proof is rather technical and appears in \Cref{app:pac-uniform}.
The high-level strategy is as follows.  For a hyperedge $E = (i_1,\ldots,i_k)$, we denote $\boldx|_E = (x_{i_1}, \ldots, x_{i_k})$.
Let $s>1$. Given a sequence $(E_1,y_1),\ldots,(E_{n^s},y_{n^s})$,
where $E_1,\ldots,E_{n^s}$ are i.i.d.\ random hyperedges, we aim to design an algorithm that distinguishes whether $\boldy=(y_1,\ldots,y_{n^s})$ is uniformly random or pseudorandom. 
Specifically, in the pseudorandom case we set $\boldy = \big(P(\boldx|_{E_1}), \ldots, P(\boldx|_{E_{n^s}})\big)$,
where $\boldx \in \{0,1\}^n$ is chosen uniformly at random. 
Assuming the existence of a tractable weak learning algorithm for regular expressions for the uniform distribution, we construct a training set that is realizable by regular expressions when $\boldy$ is pseudorandom. The hypothesis returned by the weak learner then acts as a distinguisher between the random and pseudorandom cases, thereby violating the PRG assumption.
\subsection{Comparison to \cite{angluin2013learnability,chen2014learning}}\label{sec:comparison-shuffle}
The learnability of a subclass of REs called \emph{shuffle ideals} was investigated by \citet{angluin2013learnability}. They established tractable PAC learnability under the uniform distribution, and, under cryptographic assumptions, hardness of improper PAC learning for a sufficiently large (non-binary) alphabets.
This hardness does not carry over when membership queries are allowed.
 They further noted that extending this hardness result to smaller alphabets, in particular the binary alphabet, remains an open question. Subsequently, \cite{chen2014learning} proved tractable learnability of this class for specific distributions in the statistical query model.

In contrast, we study plain regular expressions over the binary alphabet. We first establish distribution-free hardness in the PAC model. We then prove hardness under the uniform distribution, in contrast to the positive result for the subclass of shuffle ideals. Finally, in the next section, we show both distribution-free and distribution-specific hardness even in the presence of membership queries, a setting for which no such hardness results were previously known.

We note that the construction of \citet{angluin2013learnability} can be modified to yield distribution-free hardness (under cryptographic assumptions) for learning general regular expressions.  This can be done by converting their strings over a non-binary alphabet into binary strings,
incurring only a logarithmic overhead in size.  The resulting regular expressions remain valid, since we do not restrict ourselves to the subclass of shuffle ideals. This observation was not mentioned by \citet{angluin2013learnability}, who did not ask about standard REs and left the question of binary alphabets explicitly open.

\section{Hardness of MQ Learning}

In this section, we study the setting where the learner is allowed to issue membership queries, in addition to sampling random labeled examples from the distribution.

\subsection{Distribution-free hardness}\label{sec:dist-specific-ere}

%\idan{this would work only for PRG, check and update} 
The hardness of learning DNFs, discussed in Section~\ref{sec:dist-free-PAC}, extends to this setting as well. 
In particular, \citet{angluin1991wont} showed that if there exists a non-uniform one-way function that cannot be inverted by polynomial-size circuits, then the hardness of distribution-free PAC learning DNFs carries over to the model with membership queries, i.e., DNFs are either efficiently PAC-learnable from random examples alone, or else not even efficiently learnable with queries. 
Hence, using the same arguments from Section~\ref{sec:dist-free-PAC}, we have the following:

\begin{theorem}[Distribution-Free Hardness of PAC+MQ Learning REs] \label{thm:dist-free-queries}
Assume there exists a non-uniform one-way function, and moreover assume either the local-PRG assumption (Assumption~\ref{ass:local-prg}), the random $K$-SAT assumption \citep{daniely2016complexity}, or the sparse learning-parity-with-noise assumption from \citet{applebaum2025structured}.
Then, for any constant $\epsilon>0$, there is no tractable algorithm that PAC+MQ learns the concept class $$  \cC_n = \bigl\{ c : \{0,1\}^n \to \{0,1\} \bigm| c \text{ is representable by a regular expression of size}\leq n^\epsilon \bigr\}.$$ 
\end{theorem}

\subsection{Distribution-specific hardness of learning extended REs}\label{sec:dist-specific-ere}

So far we have established distribution-free hardness 
of learning regular expressions, even when membership queries are allowed. Moreover, we established hardness of learning regular expressions without queries under the uniform distribution.
In this section, we show 
hardness of learning with queries under the uniform distribution,
provided that regular expressions are extended with complement or intersection. Our approach is to reduce PAC+MQ learning Boolean formulae to the problem of learning such extended REs, from which the hardness result follows by \cite{kharitonov1993cryptographic}.

\begin{lemma}[Polynomial-Size Boolean Formula $\to$ $\REneg$\,/\,$\REcap$ Translation]\label{lem:bf-to-ere}
Let $\Sigma=\{0,1\}$ and let $\varphi(x_1,\ldots,x_n)$ be a Boolean formula of size $|\varphi|$.
There exists a regular expression $R_\varphi$ over $\Sigma$ using only union and concatenation together with either complement (for $\REneg$) or intersection (for $\REcap$), such that
$
L(R_\varphi)\cap \{0,1\}^n \;=\; \bigl\{(a_1,\ldots,a_n)\in \{0,1\}^n : \varphi(a_1,\ldots,a_n)=1\bigr\}.
$
Moreover, $|R_\varphi|=O(|\varphi|n)$, and $|R_\varphi|=O(|\varphi|\log n)$ if the counting operator is allowed.
\end{lemma}
\begin{proof} We start with the case of $\REcap$:
% \\
% \textbf{Construction with $\REcap$.}
First convert $\varphi$ to negation normal form (NNF), pushing all negations to literals. This increases the size by at most a constant factor.
Define $R(\cdot)$ by structural induction on the NNF structure and set $R_\varphi := R(\varphi)$.
For literals ($i\in[n]$),
\[
R(x_i) := (0\vee 1)^{i-1}1(0\vee 1)^{n-i},
\qquad
R(\neg x_i) := (0\vee 1)^{i-1}0(0\vee 1)^{n-i}.
\]
For connectives,
\[
R(\alpha\vee\beta) := R(\alpha) \vee R(\beta),
\qquad
R(\alpha\wedge\beta) := R(\alpha) \wedge R(\beta).
\]
By induction on the NNF structure we show
\[
L\big(R(\varphi)\big)\cap\{0,1\}^n
=\{\mathbf{a}\in\{0,1\}^n:\varphi(\mathbf{a})=1\}.
\]
The base cases hold by construction of $R(x_i)$ and $R(\neg x_i)$.
For the inductive steps, use
$L(R(\alpha\vee\beta))=L(R(\alpha))\cup L(R(\beta))$
and
$L(R(\alpha\wedge\beta))=L(R(\alpha))\cap L(R(\beta))$.
The claim follows from the induction hypothesis.

\textbf{The size of the regular expression.}
Each literal contributes $O(n)$ symbols (or $O(\log n)$ with the counting operator), and each connective adds $O(1)$.
The NNF conversion is linear-size, so $|R_\varphi|=O(|\varphi|n)$, or $O(|\varphi|\log n)$ with counting.

\textbf{Conclusion for (\texorpdfstring{$\REneg$}{REneg}).}
Since intersection is definable from complement and union by De Morgan, with constant overhead, $R(\alpha\wedge\beta) := \neg\bigl(\neg R(\alpha) \vee\neg R(\beta)\bigr)$,
the above construction immediately yields an $\REneg$ expression of the same asymptotic size. Thus, the correctness and size bounds carry over to $\REneg$ as well.
\end{proof}

The above lemma shows that polynomial-size Boolean formulas can be represented by polynomial-size REs extended with intersection or with negation. Hence, hardness of learning Boolean formulas implies hardness of learning $\REcap$ and $\REneg$.
Thus, we establish the hardness of learning extended REs via known results on Boolean formulas \citep{kharitonov1993cryptographic}. 

\begin{theorem}[Hardness of PAC+MQ Learning $\REcap$ or $\REneg$ Under the Uniform Distribution]\label{thm:ere-pac+mq-hardness}
Assuming one of the cryptographic assumptions: RSA, factoring Blum integers, or quadratic residues \citep{kharitonov1993cryptographic},
for any constant $\epsilon>0$, there is no tractable algorithm that PAC+MQ learns the concept class $$  \cC_n = \bigl\{ c : \{0,1\}^n \to \{0,1\} \bigm| c \text{ is representable by a RE with intersection or complement of size}\leq n^\epsilon \bigr\},$$
on the uniform distribution over $\{0,1\}^n$.
\end{theorem}

\begin{proof}
Under one of the assumptions: RSA, factoring Blum integers, or quadratic residues,
there exists some constant $s>1$ such that Boolean formulas of size $n^s$ cannot be tractably weakly 
learned with membership queries on the uniform distribution over $\{0,1\}^n$ \citep{kharitonov1993cryptographic}.  
By Lemma~\ref{lem:bf-to-ere}, this implies hardness of weakly learning $\REcap$ or $\REneg$ 
of size $n^{s+1}$.  
By a standard scaling argument, similar to the proof of \Cref{thm:hardness-from-DNF}, we conclude that for any $\epsilon>0$, it is hard to weakly learn such 
REs of size $n^\epsilon$.  
\end{proof}

%%%%%%%%%%%%%%%%%%%%%%%%%

\section{Discussion and Open Questions}
Regular expressions are a fundamental object in computer science, yet the computational complexity of learning them has surprisingly not been previously rigorously established.
In this work, we close this imporant gap in the literature by proving hardness results both in the PAC model and in the membership query setting, under distribution-free learning as well as under the uniform distribution.  Beyond the fact that we now have a reference to rely on regarding this hardness, closing this gap is also important since it reveals subtleties about the tractability of learning regular expressions, in particular under membership queries.

As we have discussed, learning REs is not the same as learning DFAs or NFAs, even though all three characterize the class of regular languages.
The key distinction lies in the measure of complexity, namely, the description length, which as shown in \Cref{sec:description-re-dfa}, can differ substantially across these representations.
This difference has concrete consequences for learnability: DFAs are efficiently learnable with membership queries, whereas we prove that REs remain hard in the same model.

A key takeaway is that what truly matters is the complexity measure (or description length, or equivalently `prior') induced by the model, rather than the concept class itself.
Although REs, DFAs, and NFAs all define the same family of regular languages and are thus equivalent in a uniform sense, this equivalence is largely irrelevant from a learning perspective.
For any finite $n$, every predictor defines a finite, and therefore regular language.
The real question is the quantitative {\em complexity measure} (e.g., DFA size or RE length), and these differ dramatically between different ``equivalent'' models.  Here, we focused only on polynomial hardness (a fairly crude notion), but in a more refined analysis, e.g., of sample complexity or exact runtime, finer-grained quantitative differences become even more important.

As a direction for future work, we point out that to the best of our knowledge, it is still unknown whether there exists a tractable algorithm for learning plain REs or NFAs with membership queries under the uniform distribution.  Also, while REs can be efficiently converted into NFAs, NFAs may be exponentially more succinct than REs. However, it remains unclear whether this asymmetry implies a genuine separation in their learnability.

\section*{Acknowledgments and Disclosure of Funding}
We thank Aryeh Kontorovich and Dana Angluin for helpful discussions.

This work was conducted as part of the NSF-supported Institute for Data, Econometrics, Algorithms and Learning (IDEAL), under grants NSF ECCS-2217023 and NSF EECS-2216899.
Gal Vardi is supported by the Israel Science Foundation (grant No. 2574/25), by a research grant from Mortimer Zuckerman (the Zuckerman STEM Leadership Program), and by research grants from the Center for New Scientists at the Weizmann Institute of Science, and the Shimon and Golde Picker -- Weizmann Annual Grant.

\bibliography{refs}

\newpage
\appendix

\section{Hardness of Regular Expressions in PAC Learning under the Uniform Distribution}\label{app:pac-uniform}

\begin{proof}[of \Cref{thm:re-pac-uniform}]
For $\boldx\in \{0,1\}^n$ and coordinates $(i_1,\ldots,i_\ell)$ we write $\boldx|_{(i_1,\ldots,i_\ell)} = (x_{i_1}, \ldots, x_{i_\ell})$, that is, the restriction of $\boldx$ to these coordinates. In particular, for a hyperedge $E = (i_1,\ldots,i_k)$, we write $\boldx|_E = (x_{i_1}, \ldots, x_{i_k})$.

Our high-level strategy is as follows.  
Let $s>1$. Given a sequence $(E_1,y_1),\ldots,(E_{n^s},y_{n^s})$,
where $E_1,\ldots,E_{n^s}$ are i.i.d.\ random hyperedges, we aim to design an algorithm that distinguishes whether $\boldy=(y_1,\ldots,y_{n^s})$ is uniformly random or pseudorandom. 
Specifically, in the pseudorandom case we set $\boldy = \big(P(\boldx|_{E_1}), \ldots, P(\boldx|_{E_{n^s}})\big)$,
where $\boldx \sim \{0,1\}^n$ is chosen uniformly at random. 
Assuming the existence of an efficient weak learning algorithm for regular expressions for the uniform distribution, we construct a training set that is realizable by regular expressions when $\boldy$ is pseudorandom. The hypothesis returned by the weak learner then acts as a distinguisher between the random and pseudorandom cases, thereby violating the PRG assumption.

Let $\cD$ be the uniform distribution on $\{0,1\}^{n^{\alpha}}$, where $\alpha \geq 1$ (to be determined at the end of the proof).
Suppose for contradiction that there exists an efficient algorithm $L$ that learns regular expressions of size $n\log^{2}(n)$ under $\cD$. 
Let $p(n)$ be a polynomial such that $L$ uses at most $p(n)$ samples and, with probability at least $\tfrac34$, outputs a hypothesis $h$ with error at most $\tfrac12-\gamma$.

Choose a constant $s>1$ such that $n^s \ge p(n)+n$ for all sufficiently large $n$.
By Assumption~\ref{ass:local-prg}, for every $s$, there exists a constant $k$ and a predicate $P:\{0,1\}^k \to \{0,1\}$ such that $\mathcal{F}_{P,n,n^s}$ is a $\tfrac{1}{3}$-PRG.
We design an algorithm $A$ that distinguishes random from pseudorandom with advantage greater than $\tfrac13$, contradicting the assumption. The algorithm $A$ uses $p(n)$ samples to simulate the learning algorithm $L$, along with an additional held-out set of $n$ samples that are not observed by $L$.

\paragraph{Hyperedge encodings.} 
We encode a hyperedge $E = (i_1,\ldots,i_k)$ by a vector $\boldz^E \in \{0,1\}^{kn}$, defined as the concatenation of $k$ vectors in $\{0,1\}^n$, where the $j$-th vector has a zero in the $i_j$-th coordinate and ones in all other coordinates (similar to a one-hot encoding, except that the marked entry is zero and all others are one). This construction uniquely represents the hyperedge $E$.
Let $\tilde{\boldz}^E \in \{0,1\}^{k\log(n)}$ denote the \emph{compressed encoding} of $E$, defined as the concatenation of the binary representations (each of length $\log (n)$) of the $k$ vertices of $E$. 
We say that $\tilde \boldz \in \{0,1\}^{n^{\alpha}}$ is an \emph{extended compressed encoding} of $E$ if $\tilde \boldz|_{(1,\ldots,{k\log(n)})}=\tilde \boldz^E|_{(1,\ldots,{k\log(n)})}$, that is,
the first $k\log(n)$ coordinates of $\tilde \boldz$ equal $\tilde \boldz^E$.

If $\tilde \boldz$ is uniform in $\{0,1\}^{n^{\alpha}}$, then the probability that it is a valid extended compressed encoding, namely, that each two of the first $k$ blocks of size $\log(n)$ encode different indices, is simply the probability that $k$ independently chosen indices are all distinct,
\begin{align}\label{eq:hyperedge}
\begin{split}
\Pr[\tilde \boldz \text{ encodes a hyperedge}]
&= \frac{n \cdot (n-1) \cdots (n-k+1)}{n^k} \\
&\ge \Bigl(1-\tfrac{k}{n}\Bigr)^k \\
&\ge 1-\tfrac{\gamma}{2},
\end{split}
\end{align} 
for sufficiently large $n$.
% \idan{explain that this is where we needed the compressed expression}
Note that we need the compressed encoding exactly for \eqref{eq:hyperedge}: 
if we used $k$ blocks of size $n$, then each block would have to be a one-hot 
vector, which occurs with negligible probability.
\paragraph{Simulating examples for $L$ (efficient weak learner for \re under the uniform distribution).}
For $\boldx \in \{0,1\}^n$ , let $P_{\boldx}: \{0,1\}^{kn}\rightarrow \{0,1\}$ and $\tilde P_{\boldx}: \{0,1\}^{k\log(n)} \rightarrow \{0,1\}$ be such that for every hyperedge $E$, it holds that
$P_{\boldx}(\boldz^E) = \tilde P_{\boldx}(\tilde \boldz^E) = P(\boldx|_{E})$.

Given $(E_1,y_1),\ldots,(E_{n^s},y_{n^s})$ with $E_i$ random hyperedges, 
algorithm $A$ must distinguish whether $\boldy = (y_1,\ldots,y_{n^s})$
is uniformly random or pseudorandom. In the pseudorandom case we have $\boldy = \bigl(P(\boldx|_{E_1}),\ldots,P(\boldx|_{E_{n^s}})\bigr)  
= \bigl(\tilde P_{\boldx}(\tilde \boldz^{E_1}),\ldots,\tilde P_{\boldx}(\tilde \boldz^{E_{n^s}})\bigr)$
for some uniformly random $\boldx \in \{0,1\}^n$. Denote by $\boldsymbol{E} = (\tilde \boldz^{E_1},y_1),\ldots,(\tilde \boldz^{E_{n^s}},y_{n^s})$
the corresponding encoded hyperedges sample.

Algorithm $A$ runs $L$ on oracle access to random examples drawn from $\cD$, the uniform distribution on $\{0,1\}^{n^{\alpha}}$. For each $i
\in [n^s]$, the oracle samples $\tilde \boldz_i \sim \cD$.
\begin{itemize}
    \item If $\tilde \boldz_i$ is not an extended compressed encoding (this happens with probability at most $\gamma/2$ by \cref{eq:hyperedge}), return $(\boldz_i',y_i')$ with $\boldz_i'=\tilde \boldz_i$, $y_i'=1$. 
    \item Otherwise, replace the first $k\log(n)$ coordinates of $\tilde{\boldz}_i$ with $\tilde{\boldz}^{E_i}$, and let $\boldz_i'$ denote the resulting string. Output $(\boldz_i',y_i')$, where $y_i' = y_i$. This modification preserves the uniform distribution $\boldz_i' \sim \cD$, since a random hyperedge is replaced by another random hyperedge.
\end{itemize}
Denote $S_t = \{(\boldz'_i,y'_i) : i \in [p(n)]\}$ and 
$S_v = \{(\boldz'_i,y'_i) : i \in \{p(n)+1,\ldots,p(n)+n\}\}$.
Recall that $n^s \geq p(n)+n$, and that algorithm $L$ observes only $p(n)$ samples. 
Thus, $S_t$ serves as the training set for $L$, while $S_v$ acts as a validation set.
\paragraph{Distinguisher algorithm $A$.}
Let $h$ be the hypothesis produced by $L$ given $S_t$. Denote the error of $h$ on $S_v$ by $\widehat{\err}_{S_v}(h)=\frac{1}{n}\sum_{(\boldz'_i,y'_i)\in S_v} \mathbb{I} \{h(\boldz'_i)\neq y'_i\}$.
\begin{itemize}
    \item If $\widehat{\err}_{S_v}(h) \leq \frac{1}{2}-\frac{\gamma}{2}$ then $A$ returns $1$ (pseudorandom).
    \item Otherwise $A$ returns $0$ (random).
\end{itemize}
We show below that
if $\boldsymbol{E}$ is pseudorandom, then $A$ returns $1$ with probability greater than $2/3$, and if $\boldsymbol{E}$ is random, then $A$ returns $0$ with probability greater than $2/3$.
\paragraph{Regular expression construction.} 
We construct a regular expression $R$ such that, if $\boldy$ is pseudorandom, then all examples drawn from the distribution $\cD$, and in particular the entire simulated sample $\{(\boldz'_i,y'_i) : i \in [n^s]\}$, are realizable by $R$. That is, $\forall i \in [n^s], \quad \boldz'_i \in L(R) \iff y'_i = 1$.

Fix the seed $\boldx\in\{0,1\}^n$ for the pseudorandom generator.
Without loss of generality, assume $n$ is a power of two,
% (replace $n$ by $n' = 2^{\lceil \log_2 n\rceil}$ and extend $x$
% arbitrarily on $i>n$; this affects bounds by a constant factor). 
thus, the block length
$\log(n)$ is an integer and every $\log(n)$-bit block encodes some index in $[n]$. Let $\mathrm{bin}(i)$ be the binary representation of $i\in[n]$.
For $b\in\{0,1\}$ define the 
\[
I_b := \bigvee_{i\in[n]:\ x_i=b}\ \mathrm{bin}(i),
\]
that is, $I_b$ matches exactly those $\log(n)$-bit blocks corresponding to indices $i$ with $x_i=b$.
% the union of all $\log(n)$-bit strings of indices whose $i$-th coordinate of $\boldx$ equals $b$.
For each $\boldu=(u_1,\ldots,u_k)\in\{0,1\}^k$ with $P(\boldu)=1$ set
\[
R_{\boldu} := I_{u_1}I_{u_2}\cdots I_{u_k}(0 \vee 1)^{\star}.
\]
% where $\Sigma^t$ denotes “any string of length $t$” (counting operator).
Let 
\[
R_{\boldx} := \bigvee_{\boldu:\ P(\boldu)=1} R_{\boldu}.
\]
This is a regex that accepts exactly those encodings of hyperedges $(i_1,\ldots,i_k)$ where $P(x_{i_1},\ldots,x_{i_k})=1$.

Since examples drawn from $\cD$ are not necessarily extended compressed encodings but still have label $1$, we need an additional regular expression to capture them. This occurs when an index appears more than once.
We define this duplicate-catcher regular expression as
\[
R_{\mathrm{dup}} := \bigvee_{1\le a<b\le k} \bigvee_{i\in[n]} 
(0 \vee 1)^{(a-1)\log(n)} \mathrm{bin}(i)(0 \vee 1)^{(b-a-1)\log(n)}\mathrm{bin}(i)(0 \vee 1)^{\star}.
\]
Note that we use $(0 \vee 1)^{\star}$ instead of a fixed-length suffix in both $R_{\mathrm{dup}}$ and $R_{\boldu}$, since our distribution $\cD$ is over strings of length $n^\alpha$, and the behavior on other lengths is irrelevant.
Finally, define the target regex
\[
R := R_{\boldx} \vee R_{\mathrm{dup}}.
\]

If $\boldz'_i$ encodes a valid hyperedge $E=(i_1,\ldots,i_k)$ then each block
$I_{u_j}$ enforces $x_{i_j}=u_j$. Hence $R_{\boldx}(\boldz'_i)=P(\boldx|_{E})$, and $R(\boldz'_i)=P(\boldx|_{E})$.
If the encoding is invalid (some index repeats), then $R_{\mathrm{dup}}$ accepts and $R(\boldz'_i)=1$,
matching the simulated labeling rule.
Note that the inputs to $R$ are in $\{0,1\}^{n^\alpha}$, but it uses only the first $k\log(n)$ coordinates of the input. The construction clearly runs in polynomial time. 
\paragraph{The size of $R$.} Each  $I_b$ is a union of at most $O(n)$ strings of length $\log(n)$, so $|I_b|=O(n\log(n))$. Thus, for each pattern $\boldu$ we obtain $|R_{\boldu}|=O(kn\log(n))$.
    Taking the union over at most $2^k$ such patterns gives $|R_{\boldx}|=O(2^k k n\log(n)).$
    For the duplicate--catcher, each of the $\binom{k}{2}n=O(k^2 n)$ branches. 
    Each branch has size $O(k\log (n))$: two strings of length $\log (n)$ and two fixed gaps whose total
    length is at most $(k-2)\log (n)$. Thus, $|R_{\mathrm{dup}}|
    \le O(k^{3} n \log (n))$.
    Therefore, the total size is \[|R|=|R_{\boldx}|+|R_{\mathrm{dup}}|=O(2^k k n\log(n))+O(k^{3} n \log (n)).\]
Since $k$ is independent of $n$, for sufficiently large $n$ we get $|R| \leq O(n\log^2(n))$.

\paragraph{Analysis (pseudorandom case).}
If $\boldsymbol{E}$ is pseudorandom, then for every oracle answer we have $y_i' = R(\boldz_i')$.  
Thus with probability at least $\tfrac34$, algorithm $L$ outputs $h$ with $\mathbb{E}_{\tilde \boldz \sim \cD} [\mathbb{I}\{h(\tilde \boldz)\neq R(\tilde \boldz)\}] \leq \tfrac12 - \gamma.$
Thus, $\mathbb{E}_{S_v}[\widehat{\err}_{S_v}(h)] \le \tfrac12 - \gamma$.  
By Hoeffding’s inequality, $\Pr_{S_v}\Big[\widehat{\err}_{S_v}(h) - \mathbb{E}_{S_v}[\widehat{\err}_{S_v}(h)]\ge \tfrac\gamma4\Big] \le \tfrac1{20}.$
Therefore, with probability at least $1 - \left(\tfrac14 + \tfrac1{20}\right) = \tfrac7{10} > \tfrac23,$
we get $\widehat{\err}_{S_v}(h) \leq \tfrac12 - \tfrac\gamma2$, so $A$ outputs $1$.
\paragraph{Analysis (random case).}
If $\boldsymbol{E}$ is random, then whenever $\boldz_i'$ encodes a hyperedge, the label $y_i'$ is independent of $\boldz_i'$ and uniform.  
Thus for every $h$ and $(\boldz'_i,y'_i)\in S_v$,
\begin{align*}
\Pr[h(\boldz_i') \ne y_i']
&\ge \Pr[h(\boldz_i') \ne y_i' \mid \boldz_i' \text{ encodes a hyperedge}] \cdot \Pr[\boldz_i' \text{ encodes a hyperedge}] \\
&\ge \tfrac12 \cdot (1-\tfrac\gamma2) \\
&= \tfrac12 - \tfrac\gamma4,
\end{align*}
and  $\mathbb{E}_{S_v}[\widehat{\err}_{S_v}(h)] \ge \tfrac12 - \tfrac\gamma4$.  
Applying Hoeffding again, with probability at least $\tfrac{19}{20}$ we have $\widehat{\err}_{S_v}(h) \ge \tfrac12 - \tfrac\gamma2$,
so $A$ outputs $0$ with probability at least $\tfrac23$.
\paragraph{Conclusion.}
Thus $A$ distinguishes pseudorandom from random with advantage $>1/3$, contradicting Assumption~\ref{ass:local-prg}.
Hence,  for sufficiently large $n\in \mathbb{N}$, efficient weak learning of regular expressions of size $n\log^2(n)$ is impossible under the uniform distribution over $\{0,1\}^{n^\alpha}$. We now follow a standard scaling argument (see e.g. \cite{daniely2014average} to obtain a tighter result.

Fix any $\epsilon >0$.
Choose $\alpha = 1 + \tfrac{2}{\epsilon}$, and let $\tilde{n} = n^{\alpha} = n^{1 + \tfrac{2}{\epsilon}}$. 
Thus, no polynomial-time learner weakly learns regular expressions of most $\tilde{n}^{\epsilon} = n^{2 + \epsilon}$ (which already subsumes $n\log^2 (n)$), under the uniform distribution over $\{0,1\}^{\tilde{n}}$.

\paragraph{Extensions to RE without Kleene star, and RE with the counting operator:}
\begin{itemize}
    \item \textbf{RE with union and concatenation (without Kleene star).}
    We can construct a similar star-free RE as follows. $I_b$ remains the same. $R_{\boldu}$ is modified to 
\[
R_{\boldu} := I_{u_1}I_{u_2}\cdots I_{u_k},
\] and 
\[
R_{\boldx} := \left(\bigvee_{\boldu:\ P(\boldu)=1} R_{\boldu}\right)(0 \vee 1)^{(n^\alpha-k\log(n))}.
\]
We modify $R_{\mathrm{dup}}$ as follows. 
For $1\le a<b\le k$ and $i\in[n]$ define 
\[
R_{\mathrm{dup}(a,b,i)} := 
(0 \vee 1)^{(a-1)\log(n)} \mathrm{bin}(i)(0 \vee 1)^{(b-a-1)\log(n)}\mathrm{bin}(i)
(0 \vee 1)^{(k-b)\log(n)},
\]
which means that the $a$-th and $b$-th blocks of size $\log(n)$ are the same index $i$.
Then
\[
R_{\mathrm{dup}} := 
\left(\bigvee_{1\le a<b\le k} \bigvee_{i\in[n]} 
R_{\mathrm{dup}(a,b,i)}
\right)
(0 \vee 1)^{n^\alpha-k\log(n)}.
\]
The final star-free RE is $R := R_{\boldx} \vee R_{\mathrm{dup}}.$

However, we get a larger RE. 
Each block matcher $I_b$ is a union of at most $O(n)$ strings of length $\log(n)$, so $|I_b|=O(n\log(n))$. Thus, for each pattern $\boldu$ we obtain $|R_{\boldu}|=O(kn\log(n))$.
    Taking the union over at most $2^k$ such patterns gives $|R_{\boldx}|=O(2^k k n\log(n))+ O(n^\alpha).$ Note that the suffix of size $O(n^\alpha)$ is added to $R_{\boldx}$ only once.
    For the duplicate--catcher, each of the $\binom{k}{2}n=O(k^2 n)$ branches has size $k\log(n)$, and in addition, a suffix of size $O(n^{\alpha})$. Hence, $ |R_{\mathrm{dup}}|=O(k^3 n\cdot\log(n))+O(n^{\alpha}).$
    Therefore, the total size is
    \[
    |R|=|R_{\boldx}|+|R_{\mathrm{dup}}| = O(2^k k n\log(n))+O(n^\alpha)+O(k^3 n\cdot\log(n))+O(n^{\alpha}).
    \]
    Choosing $\alpha=2$, we get $|R|=O(n^2)$. 
    
     Then, efficient weak learning of star-free regular expressions of size $O(n^{2})$ is impossible under the uniform distribution over $\{0,1\}^{n^2}$. By taking $\tilde n=n^2$, we conclude that there is no polynomial-time learner that weakly learns, under the uniform distribution over $\{0,1\}^{\tilde{n}}$, star-free regular expressions of size at most $O(\tilde n)$. Note that without the Kleene star or the counting operator, we cannot expect to get a better result than a regular expression of linear size. 
    \item \textbf{RE with union, concatenation, and counting (without Kleene star).}
    We can reduce the size of the star-free RE by using the counting operator.
    With counting, the suffix $(0\vee 1)^{n^\alpha-k\log(n)}$ is described by $O(\log (n))$ symbols, rather than $O(n^\alpha)$ repetitions and we get $|R_{\boldx}|=O(2^k kn\log(n))$.
    Similarly, for the duplicate--catcher, the suffix is also $O(\log(n))$ and $ |R_{\mathrm{dup}}|=O(k^3 n\log(n)).$
    Therefore the total size is $|R|=O(2^k kn\log(n))+ O(k^3 n\log(n))$, 
    and for sufficiently large $n$ simplifies to $|R|=O(n\log^2(n))$.
    
    The size of the regex is as in the construction of the plain RE (with Kleene star), so we conclude that there is no polynomial-time learner that weakly learns regular expressions of most $\tilde{n}^{\epsilon}$ under the uniform distribution over $\{0,1\}^{\tilde{n}}$.
    \end{itemize}

\end{proof}

\end{document}